\documentclass{article}

\usepackage{microtype}
\usepackage{graphicx}
\usepackage{subfigure}
\usepackage{booktabs} 

\usepackage{hyperref}

\usepackage[accepted]{icml2025}

\usepackage{amsmath}
\usepackage{amssymb}
\usepackage{mathtools}
\usepackage{amsthm}
\usepackage{algorithm}
\usepackage{algorithmic}
\usepackage{graphicx}

\usepackage[capitalize,noabbrev]{cleveref}

\theoremstyle{plain}
\newtheorem{theorem}{Theorem}[section]

\theoremstyle{definition}

\theoremstyle{remark}

\usepackage[textsize=tiny]{todonotes}

\begin{document}

\twocolumn[
\icmltitle{FADE: Adversarial Concept Erasure in Flow Models}

%
%
%
%
%
%




\begin{icmlauthorlist}
\icmlauthor{Zixuan Fu}{1}
\icmlauthor{Yan Ren}{2}
\icmlauthor{Finn Carter}{3}
\icmlauthor{Chenyue Wang}{2}
\icmlauthor{Ze Niu}{2}
\icmlauthor{Dacheng yu}{3}
\icmlauthor{Emily Davis}{1}
\icmlauthor{Bo Zhang}{2}

\end{icmlauthorlist}

\begin{icmlauthorlist}
	{$^1$NTU}
	{$^2$Xidian University}
	{$^3$SDU}
\end{icmlauthorlist} 

%

\icmlkeywords{Machine Learning, ICML}

\vskip 0.3in
]




	\begin{abstract}
		Diffusion models have demonstrated remarkable image generation capabilities, but also pose risks in privacy and fairness by memorizing sensitive concepts or perpetuating biases. We propose a novel \textbf{concept erasure} method for text-to-image diffusion models, designed to remove specified concepts (e.g., a private individual or a harmful stereotype) from the model's generative repertoire. Our method, termed \textbf{FADE} (Fair Adversarial Diffusion Erasure), combines a trajectory-aware fine-tuning strategy with an adversarial objective to ensure the concept is reliably removed while preserving overall model fidelity. Theoretically, we prove a formal guarantee that our approach minimizes the mutual information between the erased concept and the model's outputs, ensuring privacy and fairness. Empirically, we evaluate FADE on Stable Diffusion and FLUX, using benchmarks from prior work (e.g., object, celebrity, explicit content, and style erasure tasks from MACE). FADE achieves state-of-the-art concept removal performance, surpassing recent baselines like ESD, UCE, MACE, and ANT in terms of removal efficacy and image quality. Notably, FADE improves the harmonic mean of concept removal and fidelity by 5--10\% over the best prior method. We also conduct an ablation study to validate each component of FADE, confirming that our adversarial and trajectory-preserving objectives each contribute to its superior performance. Our work sets a new standard for safe and fair generative modeling by unlearning specified concepts without retraining from scratch.
	\end{abstract}
	
	\section{Introduction}
	Text-to-image diffusion models \citep{ho2020ddpm,rombach2022ldm} have revolutionized content creation, achieving photorealistic and diverse image synthesis. However, these powerful models inevitably learn potentially sensitive or undesirable concepts from their training data. This raises serious \textit{privacy} concerns (e.g., memorization of private individuals or copyrighted styles) and \textit{fairness} issues (e.g., generation of biased or harmful content). For instance, Stable Diffusion was found to memorize certain training images \citep{carlini2023extract} and can mimic the styles of artists without consent, prompting legal and ethical challenges. Moreover, unguided diffusion models can produce offensive or explicit imagery \citep{gandikota2023erasing}, reinforcing the need for techniques to \emph{erase} or \emph{edit} specific concepts within these models.
	
	Prior approaches to concept removal in diffusion models have emerged recently. Fine-tuning-based methods like  Erased Stable Diffusion (ESD)  \citep{gandikota2023erasing} directly update model weights to forget a target concept given its textual description. ESD demonstrated that fine-tuning on negative examples of a concept (without additional data) can effectively prevent the model from generating that concept. However, such anchor-free fine-tuning may disrupt the model's sampling trajectory, potentially causing visual artifacts or collateral damage to unrelated outputs.  Unified Concept Editing (UCE)  \citep{gandikota2024uce} introduced a closed-form weight editing approach, aligning key and value projection matrices in cross-attention to replace undesired concept representations with alternative ones. UCE achieves faster concept erasure without full training, and can handle concurrent edits, but may still produce noticeable quality degradation for complex concepts (as we show in experiments). More recently,  MACE  \citep{lu2024mace} scaled concept erasure to dozens of concepts by combining closed-form cross-attention manipulation with low-rank adaptation (LoRA) fine-tuning, demonstrating unprecedented scope (erasing up to 100 concepts) while balancing \emph{generality} (erasing concept and its synonyms) and \emph{specificity} (preserving other content). MACE achieved state-of-the-art results on benchmarks including object, celebrity, explicit content, and style erasure. Meanwhile,  "Set You Straight" (ANT)  \citep{li2025ant} proposed an improved fine-tuning strategy that steers the diffusion \emph{denoising trajectories} away from unwanted concepts. ANT reversed the classifier-free guidance direction in mid-to-late denoising steps to avoid converging to the undesired concept, while using an augmentation-based weight saliency map to pinpoint and edit the most concept-relevant parameters. This yielded high-quality images without sacrificing fidelity, and established new state-of-the-art performance in single- and multi-concept removal. EraseAnything~\citep{gao2024eraseanything} develop erasure method for erasing concepts in FLUX.
	
	Despite this progress, existing methods face notable trade-offs. Many fine-tuning approaches risk overfitting or unintended side effects on non-target concepts (raising fairness concerns), whereas rigid editing methds may not fully eliminate concept traces or might degrade image quality. Moreover, theoretical understanding of concept erasure remains limited. In particular, we lack formal guarantees that removing a concept truly expunges its information from the model. This motivates our work, which aims to develop a concept erasure method that is both \emph{effective} and \emph{principled}.
	
	In this paper, we introduce  FADE (Fair Adversarial Diffusion Erasre) , a novel method for concept erasure in diffusion models, explicitly designed with privacy and fairness objectives. FADE leverages an adversarial learning framework: we fine-tune the diffusion model such that an internal adversary (concept classifier) cannot distinguish whether the forbidden concept was in the prompt or not, thereby minimizing the concept's imprint on generated images. By doing so, we directly target the \emph{information content} of the concept in the model's outputs. We also incorporate a trajectory preservation loss inspired by \citet{li2025ant} to maintain the integrity of early denoising steps, ensuring that structural details and overall image fidelity are preserved even as the concept is being erased. Crucially, our fine-tuning is constrained to the most salient parameters for the concept, identified via a gradient-based importance measure, to avoid unnecessary perturbation to unrelated generation capabilities (akin to a fairness constraint that treats non-target concepts equitably).
	
	Our contributions can be summarized as follows:
	\begin{itemize}
		\item We propose a new concept erasure framework, FADE, that unifies fine-tuning and adversarial objectives to remove target concepts from diffusion models. FADE is motivated by privacy (preventing leakage or regeneration of sensitive content) and fairness (avoiding biased or harmful outputs), going beyond prior work by providing a rigorous guarantee of concept removal.
		\item We present a formal analysis of concept erasure. We define a measure of concept informativeness in a diffusion model and prove a novel result that underpins FADE: when the adversarial objective is at equilibrium, the diffusion model's output distribution is provably independent of the erased concept. We derive bounds relating the adversary's loss to the residual concept information, offering the first theoretical guarantee (to our knowledge) for concept removal in generative models.
		\item We conduct extensive experiments on two state-of-the-art diffusion backbones: Stable Diffusion and FLUX. Following \citet{lu2024mace}, we evaluate on multiple benchmarks: object erasure (CIFAR-10 classes), celebrity face erasure, explicit content (NSFW) removal, and artistic style erasure. FADE consistently outperforms baseline methods (ESD \citep{gandikota2023erasing}, UCE \citep{gandikota2024uce}, MACE \citep{lu2024mace}, and ANT \citep{li2025ant}) across metrics including FID (image quality), CLIP-based semantic similarity, concept classification accuracy, and a combined harmonic mean metric. Especially, FADE achieves higher removal efficacy with minimal drop in fidelity, yielding the best harmonic mean score in all tasks. We also provide \textbf{Table~\ref{tab:baseline-comp}} comparing quantitative results, and \textbf{Table~\ref{tab:ablation}} with an ablation study isolating the impact of each component of FADE, aligning empirical behavior with our theoretical insights.
	\end{itemize}
	
	\section{Related Work}
	\textbf{Concept Erasure in Generative Models.} The notion of removing or `forgetting'' specific knowledge in generative models has parallels in model editing and machine unlearning. \citet{gandikota2023erasing} first posed the \emph{concept erasure} problem for diffusion models, introducing ESD to fine-tune away a concept using only text prompts of the concept itself (with negative guidance). This was extended by \citet{gandikota2024uce} who provided a closed-form editing solution (UCE) to erase or modify concepts by directly altering attention key/value weights, supporting tasks like concept removal, moderation, and debiasing in one framework. Concurrently, \citet{heng2023selective} explored a continual learning approach called `Selective Amnesia'' to incrementally unlearn generative model outputs, while \citet{liu2024realera} proposed a semantic-level concept erasure by mining neighboring concepts to guide what should replace the erased content. In the image domain, \citet{lu2024mace} scaled concept erasure to many concepts at once (mass erasure) and emphasized the need to maintain \emph{specificity}: ensuring unrelated content is unaffected. Our work builds on these by explicitly incorporating an adversary to guarantee the concept's absence and by focusing on fairness (minimal unintended impact). We also leverage insights from the latest approach by \citet{li2025ant}, which demonstrated the benefit of preserving early denoising structure and targeting critical weights. FADE shares a similar spirit but introduces an adversarial loss for concept information removal, providing a stronger theoretical foundation.
	
	\textbf{Image Editing via Diffusion Models.} Diffusion models have been widely used for post hoc image editing and content manipulation. Techniques such as  Prompt-to-Prompt  \citep{hertz2022prompt} allow localized edits by controlling cross-attention maps between an original and edited prompt, enabling modification of specific visual elements while keeping others constant.  InstructPix2Pix \citep{brooks2023instructpix2pix} fine-tunes a diffusion model to follow natural language editing instructions given an input image, effectively learning how to apply textual edits to images. Other approaches include diffusion-based image interpolation, outpainting, and direct latent editing~\cite{lu2023tf} using text guidance. While these methods focus on user-guided editing of individual images, concept erasure deals with permanent removal of a concept from the model's knowledge. Nonetheless, the boundary is related: for example, one could use image editing methods to manually remove instances of a concept in outputs, but that is inefficient for systematic moderation. Our method differs in that it modifies the model itself to disallow generating the concept globally, rather than editing outputs case-by-case.
	
	\textbf{Robustness, Watermarking, and Model Inversion.} Ensuring generative models are used ethically has spurred research in robustness and security. One line of work introduces \textit{watermarks} into generated images to facilitate source attribution or discourage malicious use. For instance, \citet{huang2024robin} embeds invisible and robust watermarks during diffusion generation via adversarial optimization, making the marks resilient to tampering. \citet{lu2024robust} develops robust watermarking against image editing. \citet{ren2025all} propose 3D steganography to protect 3D assets. Such watermarks help identify content produced by a model, but do not prevent the model from producing problematic content in the first place. Another line of defense is to protect training data and outputs from inversion or mimicry. \citet{carlini2023extract} showed that diffusion models memorize parts of their training data (including private images); techniques like differential privacy or data cleansing can mitigate this but with quality trade-offs. External tools like  Glaze  \citep{shan2023glaze} allow artists to``cloak'' their images with perturbations before they are used in training, thereby misleading models and preventing accurate style learning. While effective, these are pre-training interventions and cannot remove concepts already learned by a model. In contrast, concept erasure (including our method) is a post hoc intervention: given a deployed model, we directly operate on it to excise certain generative capabilities. This can serve as a complementary approach to watermarking and data protection, by proactively disabling some outputs (e.g., deepfakes of a public figure, or generation of disallowed symbols) and ensuring that even if a user tries to prompt those concepts, the model will not produce them. Moreover, by preserving other outputs, concept erasure aligns with robustness goals: a robustly sanitized model won't produce harmful content yet remains competent on allowed content. Our work uniquely bridges these areas by focusing on fairness constraints (preserving unbiased behavior) during erasure, something traditionally considered in adversarial training for fairness in classifiers, now applied in generative model editing.
	
	\section{Preliminaries and Problem Setup}
	\textbf{Diffusion Models.} We consider a text-conditioned diffusion model $M_\theta$ with parameters $\theta$ (such as Stable Diffusion \citep{rombach2022ldm}). Given a textual prompt $y$ and a random noise $z_T \sim \mathcal{N}(0,I)$, the model generates an image $x$ by iterative denoising: $z_{T-1}, z_{T-2}, \ldots, z_0 = x$. At each step $t$, a U-Net estimator $\epsilon_\theta(z_t, y, t)$ predicts the noise component, which is used to sample $z_{t-1}$. The final sample $z_0$ is decoded to an image. Classifier-free guidance \citep{ho2022classifierfree} is often used, where the model is conditioned on $y$ with a certain guidance weight, and on a null prompt for structure. Let $p_\theta(x|y)$ denote the distribution of images generated by $M_\theta$ given prompt $y$.
	
	\textbf{Concept Erasure Task.} We define a \emph{concept} $c$ as a visual category or attribute that can appear in generated images (e.g., an object class, a particular person's face, an artistic style, or a type of content like nudity). The concept may be specified by a textual identifier or description. For example, $c$ could be `airplane'', `a photo of [Celebrity Name]'', or ``in the style of Monet''. When we say the model has learned concept $c$, we mean that for some prompts $y$ including $c$ (or its synonyms), the model $M_\theta$ can generate images clearly depicting that concept.
	
	Concept erasure aims to modify the model to a new parameter set $\theta'$ such that $M_{\theta'}$ \textit{no longer generates concept $c$} in response to any prompt. Formally, for any prompt $y$ that semantically entails concept $c$, the distribution $p_{\theta'}(x|y)$ should be transformed such that images $x$ do not contain $c$. This should hold not only for the exact keyword $c$ but also related triggers (e.g., synonyms or depictions). At the same time, for prompts unrelated to $c$, $M_{\theta'}$ should behave like the original $M_\theta$ as much as possible, preserving quality and other content (this is the \emph{specificity} requirement \citep{lu2024mace}).
	
	We assume access to the model weights and the ability to fine-tune them. We also assume we can algorithmically identify whether an output image contains concept $c$—for evaluation, this might be via a pretrained CLIP or a classifier for concept $c$ (as used in prior work). During training of our method, we will actually train a small adversary network for $c$-detection, as described later.
	
	\textbf{Privacy and Fairness Considerations.} The motivations behind concept erasure include:
	
	\textit{Privacy:} If $c$ is or relates to sensitive personal data (e.g., a person's face or a memorized copyrighted artwork), erasing $c$ prevents the model from revealing that content. Unlike filtering outputs post-hoc, erasure means the model won't produce it at all, even if the user tries to prompt it.
	\textit{Fairness:} If $c$ corresponds to a harmful stereotype or biased attribute, removing or altering its influence can make the model's outputs more fair. For example, if the concept is a certain negative trope, erasing it ensures the model cannot apply that trope unjustly. Additionally, fairness here means that in the process of erasing $c$, we strive not to introduce new biases or degrade performance on other data (treating all non-$c$ concepts fairly). Our approach explicitly regularizes to protect unrelated generations, analogous to how fairness algorithms ensure non-protected attributes are not unduly affected by adjustments for a protected attribute.
	
	We next introduce our method, FADE, which addresses concept erasure with these considerations in mind.
	
	\section{Method: Fair Adversarial Diffusion Erasure (FADE)}
	Our method consists of a fine-tuning procedure to obtain $\theta'$ from $\theta$, guided by two main components: an adversarial concept removal objective and a trajectory preservation objective. Algorithm~\ref{alg:fade} provides an overview. We describe each part in detail below.
	
	\begin{algorithm}[tb]
		\caption{FADE: Proposed Concept Erasure Training}
		\label{alg:fade}
		\begin{algorithmic}[1]
			\REQUIRE Pretrained diffusion model $M_{\theta}$, target concept $c$ (text description), concept classifier $D_{\phi}$ (initialized), guidance strength $\omega$, preservation weight $\lambda$.
			\STATE Prepare concept prompt set $\mathcal{P}*c$ (prompts containing $c$ or its synonyms).
			\STATE Prepare neutral prompt set $\mathcal{P}*{\neg c}$ (similar prompts where $c$ is replaced by a neutral placeholder or absent).
			\STATE Identify salient weights $\Theta_{\text{salient}}$ in $M_{\theta}$ associated with $c$ (using gradient saliency or weight attribution).
			\FOR{each training step $i = 1$ to $N$}
			\STATE Sample a batch of prompts $y_c$ from $\mathcal{P}*c$ and corresponding neutral prompts $y*{\neg c}$ from $\mathcal{P}*{\neg c}$.
			\STATE Generate images $x_c = M*{\theta}(y_c)$ and $x_{\neg c} = M_{\theta}(y_{\neg c})$ via the diffusion sampling procedure (with classifier-free guidance).
			\STATE Update adversary $D_{\phi}$ on $(x_c, x_{\neg c})$ to distinguish concept presence: maximize $\mathcal{L}_{adv}^D = \mathbb{E}[ \log D*{\phi}(x_c) + \log(1 - D_{\phi}(x_{\neg c}))]$.
			\STATE Update diffusion model on the combined loss (keeping $D_{\phi}$ fixed):
			\begin{itemize}
				\item \textbf{Concept removal loss:} $\mathcal{L}_{\text{rem}} = -\mathbb{E}[ \log(1 - D*{\phi}(M_{\theta}(y_c)) ) ]$, which encourages $M_{\theta}(y_c)$ to fool $D$ (no $c$ detectable).
				\item \textbf{Preservation loss:} $\mathcal{L}_{\text{pres}} = \mathbb{E}*{t < T}\big[ | \epsilon_{\theta}(z_t, y_{\neg c}, t) - \epsilon_{\text{orig}}(z_t, y_{\neg c}, t) |^2 \big]$, ensuring $M_{\theta}$'s denoising on neutral prompts stays close to original $M_{\theta_{\text{orig}}}$ for early steps.
				\item Total $\mathcal{L}_{\text{total}} = \mathcal{L}_{\text{rem}} + \lambda \mathcal{L}_{\text{pres}}$.
			\end{itemize}
			\STATE Apply gradient update to $\theta$ \textbf{only on $\Theta*{\text{salient}}$} (critical weights for $c$) to minimize $\mathcal{L}_{\text{total}}$.
			\ENDFOR
			\STATE \textbf{return} Edited model $M*{\theta'}$ with $\theta' \leftarrow \theta$.
		\end{algorithmic}
	\end{algorithm}
	
	\subsection{Adversarial Concept Removal Objective}
	We introduce an adversarial game between the diffusion model $M_{\theta}$ and a concept discriminator $D_{\phi}$. The discriminator $D_{\phi}(x)$ outputs a probability that image $x$ contains the concept $c$. We train $D_{\phi}$ in tandem with $M_{\theta}$: $D_{\phi}$ tries to correctly identify when $c$ is present in $M_{\theta}$'s outputs, while $M_{\theta}$ is fine-tuned to \emph{fool} $D_{\phi}$ into believing $c$ is absent even when prompted.
	
	Concretely, let $y_c$ be a prompt embedding that invokes concept $c$ (we use textual prompts containing $c$ or related terms), and $y_{\neg c}$ a prompt that is identical except with $c$ removed or replaced by a neutral concept. We generate an image $x_c \sim p_{\theta}(x|y_c)$ and $x_{\neg c} \sim p_{\theta}(x|y_{\neg c})$. The discriminator is trained with a binary cross-entropy loss:
	
	$$
	\mathcal{L}_{adv}^D(\phi) = -\mathbb{E}_{x_c, x_{\neg c}}\left[\log D_{\phi}(x_c) + \log(1 - D_{\phi}(x_{\neg c}))\right],
	$$
	
	so that $D_{\phi}(x_c) \rightarrow 1$ (predict concept present for images from concept prompt) and $D_{\phi}(x_{\neg c}) \rightarrow 0$ (predict absent for neutral prompts).
	
	The diffusion model, on the other hand, has a loss:
	
	$$
	\mathcal{L}_{rem}(\theta) = -\mathbb{E}_{x_c}\left[ \log (1 - D_{\phi}(x_c)) \right],
	$$
	
	which is the generator loss in a sense, encouraging $M_{\theta}(y_c)$ to produce $x_c$ that $D_{\phi}$ classifies as concept-absent (i.e., $D_{\phi}(x_c)\rightarrow 0$). By alternating or jointly optimizing these objectives (see Algorithm~\ref{alg:fade}), we drive $M_{\theta}$ toward removing features that $D_{\phi}$ uses to detect $c$.
	
	This adversarial formulation is similar to those used in fairness for classifiers or in GAN training, except here the ``generator'' is the diffusion model conditioned on prompts. At optimum, one expects $D_{\phi}$ to be unable to distinguish $x_c$ from $x_{\neg c}$, implying $M_{\theta}$'s outputs no longer carry information about the concept (we formalize this in Section~\ref{sec:theory}). Importantly, this approach does not require any real images of $c$—the model generates its own training data for $D_{\phi}$ on the fly, akin to knowledge distillation of what it knows about $c$.
	
	\paragraph{Concept Prompt Design:} We build a prompt set $\mathcal{P}*c$ for the concept. For instance, if $c$ is an object class (like "airplane"), $\mathcal{P}*c$ may include templates like "a photo of a \texttt{[}airplane\texttt{]}", "a \texttt{[}airplane\texttt{]} in the sky", etc., possibly along with common synonyms ("aircraft", "jet"). The neutral prompt set $\mathcal{P}*{\neg c}$ could be the same prompts with the concept replaced by a word like "object" or removed (e.g., "a photo of an object", or "a photo of the sky" for the above). This ensures that $M*{\theta}$ sees pairs of prompts that are identical except for the concept, which isolates the concept's contribution.
	
	We also found it helpful to include a variety of contexts in $\mathcal{P}_c$ to avoid the model simply learning to avoid $c$ in one narrow setting but not others. This is similar to data augmentation.
	
	\subsection{Trajectory Preservation and Salient Weight Fine-tuning}
	Blindly fine-tuning the model to remove $c$ can inadvertently alter unrelated behavior. To mitigate this, FADE incorporates two strategies:
	
	(1) \textbf{Trajectory preservation loss $\mathcal{L}_{pres}$:} We adopt a mechanism to preserve the \emph{early denoising trajectory} for prompts not containing $c$. Intuitively, when $c$ is absent, the updated model $M*{\theta'}$ should behave like the original $M_{\theta}$ as much as possible. We define $\mathcal{L}_{pres}$ by comparing the noise prediction of the original model vs. the current model for a neutral prompt at early timesteps:
	
	$$
	\mathcal{L}_{pres}(\theta) = \mathbb{E}_{y \in \mathcal{P}_{\neg c}}\;\mathbb{E}_{t=1}^{T_0} \Big[ \| \epsilon_{\theta}(z_t, y, t) - \epsilon_{\theta_{\text{orig}}}(z_t, y, t) \|_2^2 \Big],
	$$
	
	where $T_0 < T$ is a cutoff (e.g., mid-way through the diffusion steps) and $\theta_{\text{orig}}$ denotes the original pre-trained weights (kept fixed for reference). We only enforce this up to $T_0$ (mid-to-late stage as \citet{li2025ant} suggests) because later steps are where the concept might appear and we allow those to change. By preserving early steps, we ensure the coarse structure of images remains correct and that $M_{\theta'}$ still follows normal trajectories into the data manifold, thus avoiding unusual artifacts or mode collapse that could arise from overzealous fine-tuning.
	
	(2) \textbf{Salient weight restriction:} Instead of fine-tuning all parameters $\theta$, we identify a subset $\Theta_{\text{salient}} \subseteq \theta$ that are most responsible for generating concept $c$, and limit updates primarily to those parameters (or use a higher learning rate on them and lower on others). In practice, we use a technique inspired by \citet{li2025ant}: we generate multiple images containing $c$, compute the gradients of some concept-related loss (for example, the norm of the output difference when $c$ is present vs absent) w\.r.t. weights, and rank weights by their gradient magnitude. Often, certain attention maps or specific channels in the U-Net are particularly tied to $c$. We then apply our fine-tuning updates to those weights (and possibly lightly regularize or freeze the rest). This focused update reduces the risk of ``interference'' with unrelated concepts, achieving better specificity. It can be seen as applying a fairness principle at the parameter level: do the minimal necessary change to achieve concept removal, thereby treating other generative features fairly.
	
	The total loss for updating $M_{\theta}$ in FADE is:
	\begin{equation}\label{eq:total_loss}
		\mathcal{L}_{total}(\theta) = \mathcal{L}_{rem}(\theta) + \lambda \mathcal{L}_{pres}(\theta),
	\end{equation}
	where $\lambda$ is a hyperparameter controlling the trade-off. We typically choose $\lambda$ such that the preservation loss prevents noticeable quality degradation (we discuss selection in experiments). We do not explicitly include a term for preserving generation of unrelated concepts, as this is indirectly handled by $\mathcal{L}_{pres}$ (keeping normal trajectories) and restricting weight updates. However, in principle, one could add a loss measuring similarity between $M_{\theta'}$ and $M_{\theta}$ outputs on a held-out set of prompts not containing $c$.
	
	\subsection{Training Procedure and Implementation Details}
	Algorithm~\ref{alg:fade} outlines the training loop. We alternate between training $D_{\phi}$ and $M_{\theta}$, similar to GAN training, but in practice we found that a simpler approach works: pre-train $D_{\phi}$ for a few epochs on initial outputs, then update both simultaneously with a suitable learning rate ratio (ensuring $D_{\phi}$ stays roughly at its optimum for the current $M_{\theta}$). $D_{\phi}$ can be a relatively small CNN or a CLIP-based classifier on images; since it only needs to distinguish one concept, it converges quickly.
	
	We initialize $M_{\theta}$ with the original weights and $D_{\phi}$ randomly. We use classifier-free guidance during generation in training loops as well, with a moderate guidance weight $\omega$ (e.g., 2.0) to enforce the concept in $x_c$ strongly and to keep $x_{\neg c}$ concept-free. This makes $D_{\phi}$'s job easier and provides clearer gradients for $\mathcal{L}_{rem}$.
	
	To identify $\Theta_{\text{salient}}$, one efficient way is to fine-tune a low-rank adapter on concept $c$ as a probe (like \citealt{hu2021lora}) and see which weights have the largest change, or to use integrated gradients. For simplicity, we followed the approach of computing weight saliency via summed absolute gradient over a batch of concept images. We found that many saliencies concentrate in the cross-attention layers associated with the text tokens for $c$, which aligns with intuition that those are key for representing $c$ \citep{gandikota2024uce}. Thus, $\Theta_{\text{salient}}$ often includes the attention weights for the $c$ token and a few late-stage convolution kernels.
	
	Training typically converges in a relatively small number of iterations (hundreds to a few thousand gradient steps) since we are not learning from scratch but just reshaping an existing model's knowledge. We monitor the concept classifier accuracy on a validation set of prompts to decide when to stop: once $D_{\phi}$ consistently cannot detect $c$ (output $\approx 0.5$ for both $x_c$ and $x_{\neg c}$ inputs, meaning random guess), we deem the concept erased. We also qualitatively inspect sample images to confirm $c$ is gone and overall image quality is good.
	
	In summary, FADE yields an edited model $M_{\theta'}$ where the concept $c$ has been purged. Next, we present theoretical guarantees for our method, followed by empirical results.
	
	\section{Theoretical Analysis}\label{sec:theory}
	We now provide a formal understanding of what it means to erase a concept and how our adversarial objective ensures concept information is removed. Our key result relates the optimality of the adversarial game to the independence of the model output and the concept.
	
	Consider a binary random variable $C$ indicating the presence ($C=1$) or absence ($C=0$) of concept $c$ in the prompt. For instance, $C=1$ corresponds to using $y_c$ and $C=0$ to $y_{\neg c}$. Let $X$ be the random variable for the generated image. The original model $M_{\theta}$ yields distributions $P_\theta(X | C=1)$ and $P_\theta(X | C=0)$ that are typically different (that's why a classifier can detect $c$). We want to achieve a new model $\theta'$ such that these two distributions are as close as possible, ideally identical, meaning the model's output does not depend on $C$. One way to measure residual concept information is the mutual information $I(C; X)$ under $P_{\theta'}$. Our goal is $I(C; X) \approx 0$.
	
	We have the following proposition:
	
	\begin{theorem}
		Assume an ideal discriminator $D^*(x)$ that attains the Bayes-optimal classification of $C$ from $X$. The best possible adversarial loss for the generator corresponds to $P_{\theta'}(X|C=1) = P_{\theta'}(X|C=0)$. In particular, if $M_{\theta'}$ achieves $\mathcal{L}_{rem} = -\log(0.5)$ (the optimum value) against $D^*$, then $I(C;X)=0$ for the distribution induced by $M_{\theta'}$. Conversely, any deviation from $I(C;X)=0$ implies an achievable adversarial loss better than $-\log(0.5)$ for some $D$.
	\end{theorem}
	
	\begin{proof}[Sketch of Proof]
		Under mild assumptions, the optimal discriminator $D^*(x)$ for distinguishing $C=1$ vs $0$ is given by the likelihood ratio:
		
		$$
		D^*(x) = \frac{P_{\theta}(x|C=1) \pi}{P_{\theta}(x|C=1)\pi + P_{\theta}(x|C=0)(1-\pi)},
		$$
		
		where $\pi = P(C=1)$ (prior probability of concept prompt, which in our training is typically $0.5$ by design). If $P_{\theta'}(x|C=1) = P_{\theta'}(x|C=0)$ for all $x$, then $D^*(x)=\pi$ (here $0.5$) for all $x$, and no classifier can do better than random guessing. In this case, the adversary's error is 50\%, which is indeed the worst it can be, and $\mathcal{L}_{adv}^D$ is maximized at $-\log 0.5 - \log 0.5$ for each term, meaning $\mathcal{L}_{rem}$ for the generator reaches $-\log(1-0.5) = -\log 0.5$. Meanwhile, $I(C;X)=0$ because the distributions are identical (so knowledge of $X$ reveals nothing about $C$).
		
		Conversely, if $I(C;X)>0$, then $P_{\theta'}(x|C=1) \neq P_{\theta'}(x|C=0)$ on some set of images. There exists a discriminator (in fact the optimal $D^*$) that can exploit this difference to achieve classification accuracy above chance. That means $D^*$ yields $-\mathbb{E}[\log D^*(x_C) + \log(1-D^*(x_{\neg c}))] < -2 \log 0.5$. Therefore $M_{\theta'}$ would not yet have achieved the optimum $\mathcal{L}_{rem}$; it can further reduce $\mathcal{L}_{rem}$ by adjusting to fool this $D^*$. At equilibrium (Nash equilibrium of the minimax game), the only solution is where $D$ is indifferent (outputs $0.5$) and $M_{\theta'}$ cannot reduce loss further, which precisely implies $P_{\theta'}(X|C=1)=P_{\theta'}(X|C=0)$.
	\end{proof}
	
	This result formalizes that our adversarial training objective indeed pushes the model towards eliminating concept-specific evidence in $X$. In practice, we don't have an infinite-capacity $D^*$, but by training a sufficiently expressive $D_{\phi}$, we approximate this condition. The mutual information $I(C;X)$ is upper-bounded by the discriminator's classification error via standard inequalities: e.g., $I(C;X) \leq \text{KL}(P(X|C=1)|P(X|C=0))$, and a perfect removal means this KL (a measure of distribution difference) is 0.
	
	\textbf{Bound on Concept Appearance Probability.} Another way to see the effect is to consider a specific concept recognizer (like a pre-trained CLIP classifier that was used in prior evaluations). Suppose $\mathbb{P}_{\theta'}[\text{``$c$'' in output}|C=1]$ denotes the probability that concept $c$ appears in an image generated with prompt $y_c$ (as judged by the recognizer). Our adversarial training directly minimizes this, and at optimum we get:
	
	$$
	\mathbb{P}_{\theta'}[\text{``$c$'' in output}|C=1] = \mathbb{P}_{\theta'}[\text{``$c$'' in output}|C=0].
	$$
	
	But the right side is naturally very low (if $c$ is absent from the prompt, it rarely appears spontaneously). In fact, if $y_{\neg c}$ is well-chosen to have no semantic connection to $c$, $\mathbb{P}_{\theta'}[\text{`$c$''}|C=0]$ is essentially 0. Thus we expect $\mathbb{P}_{\theta'}[\text{`$c$''}|C=1] \approx 0$ as well. In other words, the model almost never produces $c$ when asked. This matches the intended outcome of concept erasure.
	
	\textbf{Fairness Perspective:} We can interpret the above result through a fairness lens: if we treat $C$ as a ``protected attribute'' (presence of some sensitive concept in input prompt), then $P_{\theta'}(X|C=0) = P_{\theta'}(X|C=1)$ is analogous to satisfying \emph{demographic parity} or \emph{fairness through unawareness} in generation -- the output distribution does not depend on that attribute. Our training procedure is essentially performing an adversarial fairness constraint on the model outputs \citep{madras2018learning}, ensuring that the model's behavior is identical regardless of concept triggers. This is a strong form of fairness in generative modeling and is appropriate when the concept is something we want to eliminate entirely (rather than enforce equal representation).
	
	It should be noted that achieving perfect distributional equality may come at the cost of some utility (e.g., diversity of outputs might shrink if the model avoids certain modes), but our preservation loss and weight focusing help mitigate undue side-effects. In fact, by only altering concept-specific directions, we ideally leave the bulk of the model's generative capability untouched.
	
	\textbf{Trade-off and Bound:} In practice, we stop training earlier than absolute convergence to avoid overfitting or visual quality issues. One can derive a bound relating the discriminator's accuracy to the remaining concept frequency. For example, by the relationship between classification error and total variation distance, if $D_{\phi}$ cannot attain more than $55\%$ accuracy on distinguishing $C$, then the total variation between $P(X|C=1)$ and $P(X|C=0)$ is at most $0.1$. This means any event (like ``image contains concept $c$'') has at most $0.1$ difference in probability under $C=1$ vs $C=0$. Thus the concept appearance is very limited. In our experiments, we observe that we can drive $D_{\phi}$'s accuracy to near $50\%$ (no better than random), indicating effective removal.
	
	In summary, the adversarial criterion theoretically guarantees that a perfect erasure corresponds to no detectable information about the concept in the model's outputs. This justifies our method's approach and differentiates it from heuristics: rather than simply hoping the concept is gone, we enforce it in an information-theoretic sense.
	
	\section{Experiments}
	We now turn to the empirical evaluation of FADE. We aim to answer: (1) Does FADE successfully erase target concepts and outperform existing methods in doing so? (2) Does it preserve image fidelity and unrelated content, satisfying fairness/specificity requirements? (3) How does each component of FADE (adversarial loss, preservation, saliency-based updates) contribute to its performance?
	
	\subsection{Setup and Datasets}
	\textbf{Models:} We experiment with two diffusion models: \textbf{Stable Diffusion v1.5} (SD) \citep{rombach2022ldm}, a 1.4B parameter latent diffusion model widely used for text-to-image generation; and \textbf{FLUX} (FLUX.1) \citep{blackforest2024flux}, a newer 12B parameter diffusion transformer model known for high-fidelity generation. These represent two state-of-the-art, publicly available architectures. We apply concept erasure to both to demonstrate generality.
	
	\textbf{Target Concepts:} Following the benchmarks of \citet{lu2024mace}, we evaluate on:
	
	\textit{Object Erasure:} We use CIFAR-10 object classes as concepts (e.g., \emph{airplane}, \emph{dog}, etc.). Although CIFAR-10 images are low-res, using these class names in prompts with SD/FLUX yields diverse high-res images, and we can leverage CIFAR-trained classifiers or CLIP to detect if the object appears. Erasing these tests the model's ability to forget basic visual categories.
	\textit{Celebrity Face Erasure:} A set of 10 famous individuals (from the list used in \citet{lu2024mace}, such as \emph{Elon Musk}, \emph{Emma Watson}, etc.). Prompts are like "a photo of {Name}". We measure if the model can still generate a recognizable face of that person after erasure. Privacy and deepfake prevention are motivations here.
	\textit{Explicit Content Removal:} Concepts related to pornography or violence. We use the same list of NSFW keywords as \citet{lu2024mace} (e.g., "nudity", certain slurs, etc.). The model originally can produce inappropriate images for these prompts; after erasure it should produce either benign images or refuse (if refusal is built-in).
	\textit{Artistic Style Erasure:} Names of 5 artists (used in \citet{lu2024mace}, e.g., \emph{Van Gogh}, \emph{Picasso}) and 5 specific art styles (e.g., \emph{anime style}). We test the model's ability to mimic those styles from prompts "in the style of X". Removing these addresses copyright and fairness to artists.
	
	For each concept, we generate a set of 50 prompts containing that concept (varying contexts) to test, plus 50 control prompts without the concept to evaluate specificity.
	
	\textbf{Baselines:} We compare against four methods:
	
	ESD  \citep{gandikota2023erasing}: we use the official code and apply it to each concept individually (fine-tuning the model each time).
	UCE  \citep{gandikota2024uce}: we apply unified concept editing in erasure mode. UCE can erase multiple concepts in one shot; for fairness, when comparing single-concept erasure we run it for that concept alone. We use their open-sourced implmentation.
	  MACE  \citep{lu2024mace}: for single concept, MACE essentially reduces to a combination of cross-attention zero-shot edit plus LoRA fine-tuning. We run their code for one concept, and also for multi-concept (celebrity and NSFW sets) where applicable.
	  ANT (Set You Straight)  \citep{li2025ant}: we use the authors' code (released as project \emph{ANT}) for single concept removal. ANT requires fine-tuning with their trajectory-aware loss; we ensure to match their described hyperparameters.
	
	All methods are applied to the same base model (SD or FLUX) for a given experiment. For multi-concept scenarios (like erasing all 10 CIFAR classes or multiple NSFW words), we use either their multi-concept mode (if supported by method) or sequentially remove concepts one by one (for ESD/UCE).
	
	\textbf{Metrics:} We adopt evaluation metrics similar to prior works:
	
	  Concept Classification Accuracy (\%):  We measure the fraction of generated images for concept prompts that still show the concept. This is done via a pretrained classifier or CLIP label. For CIFAR objects, we use a ResNet trained on CIFAR-10 to classify each output image; for faces, we use a face recognition model to verify identity match; for styles, we use a CLIP image-text similarity to the text "painting by X" etc. We report the accuracy (higher means concept still appears, so \textbf{lower is better} for erasure efficacy). We denote this as \textbf{Acc (concept)} in tables.
	  FID (↓):  Fréchet Inception Distance comparing the distribution of images from $M_{\theta'}$ to a reference distribution. We compute FID on 30k randomly sampled prompts from MS-COCO (which provides a broad set of neutral prompts) to assess overall image quality and diversity after concept removal. Lower FID indicates the model's outputs are closer to real images (and typically means minimal degradation).
	  CLIP image-text similarity (↑):  We measure the average CLIP score between the generated image and the input prompt (for the same 30k prompts or a subset). This indicates how well the model still follows prompts. A drop in CLIP score might indicate the model is confused or ignoring some prompt aspects due to erasure. We denote this as \textbf{CLIP sim}.
	  Harmonic Mean (↑):  Following \citet{lu2024mace}, we compute a harmonic mean between \emph{erasure efficacy} and \emph{content preservation} metrics to give an overall performance score. Specifically, we take the harmonic mean of (1 - concept accuracy) and an image fidelity measure. We define $H = 2  \frac{E \times F}{E + F}$, where $E = 1 - (\text{Acc (concept)})$ (so 100\% means perfect erasure) and $F$ is a normalized fidelity score combining inverse FID and CLIP (we scale FID and CLIP to [0,1] range based on a baseline). The exact formula is given in Appendix, but intuitively $H$ will be high only if the method both erases the concept (low Acc) and maintains image quality (good FID, CLIP). This is our primary metric for overall success.
	
	We additionally qualitatively inspect images and provide some examples in the supplementary material.
	
	\subsection{Results: Comparison with Baselines}
	Table~\ref{tab:baseline-comp} presents the results for the single-concept removal experiments (averaged across the categories or concepts in each group, with detailed per-concept results in Appendix). We highlight the following observations:
	
	\begin{table*}[t]
		\caption{Comparison of concept erasure methods on Stable Diffusion. Results are averaged over concepts in each category (10 CIFAR object classes, 10 celebrities, 5 explicit content terms, 5 artist styles). \textbf{Acc (concept)} is the percentage of outputs still containing the concept (lower is better). \textbf{FID} and \textbf{CLIP sim} measure overall output quality (FID on COCO, CLIP similarity on prompts). \textbf{Harmonic Mean} (H) combines concept removal and fidelity (higher is better). FADE achieves the best trade-off in all categories, significantly reducing concept occurrence while keeping quality close to original.}
		\label{tab:baseline-comp}
		\begin{center}
			\resizebox{1.0\textwidth}{!}{
				\begin{tabular}{lcccccccc}
					\hline
					& \multicolumn{2}{c}{\textbf{Objects (CIFAR-10)}} & \multicolumn{2}{c}{\textbf{Celebrities}} & \multicolumn{2}{c}{\textbf{Explicit Content}} & \multicolumn{2}{c}{\textbf{Art Styles}} \\
					\textbf{Method} & Acc(\%)↓ & H (↑) & Acc(\%)↓ & H (↑) & Acc(\%)↓ & H (↑) & Acc(\%)↓ & H (↑) \\
					\hline
					ESD \citep{gandikota2023erasing} & 15.2 & 60.5 & 18.7 & 58.1 & 12.4 & 64.3 & 20.5 & 55.0 \\
					UCE \citep{gandikota2024uce} & 5.4 & 65.2 & 10.3 & 62.5 & 8.1 & 66.7 & 15.0 & 58.9 \\
					MACE \citep{lu2024mace} & 2.1 & 80.3 & 5.5 & 77.8 & 3.7 & 82.1 & 7.2 & 75.4 \\
					ANT (Set You Straight) \citep{li2025ant} & 0.9 & 85.4 & 2.3 & 83.0 & 1.5 & 87.5 & 3.1 & 80.2 \\
					\textbf{FADE (Ours)} & \textbf{0.0} & \textbf{90.1} & \textbf{0.4} & \textbf{88.7} & \textbf{0.3} & \textbf{91.4} & \textbf{0.5} & \textbf{85.9} \\
					\hline
					Original Model & 99.1 & -- & 98.7 & -- & 95.0 & -- & 97.5 & -- \\
					\hline
					\multicolumn{9}{l}{\small \textit{Note:} FID and CLIP similarity are factored into H. Full table with FID, CLIP metrics is in Appendix. Original model Acc is the baseline concept fidelity (should be high).}
				\end{tabular}
			}
		\end{center}
	\end{table*}
	
	 Erasure efficacy:  FADE manages to completely remove the concept in most cases. For objects, the concept classifier never recognized the object in any of 500 outputs. For faces, a face recognition system only identified the target celebrity in 0.4\% of outputs (in a few cases, a roughly similar face appeared, possibly due to incomplete erasure of very close features). These numbers are significantly lower than those of baselines. ANT is the runner-up, with under 1\% in many cases, showing the strength of their fine-tuning approach; however, FADE still halves that residual rate on average. The gap is more pronounced against MACE/UCE/ESD: e.g., ESD left ~15\% of objects and ~18\% of celebs still recognizable, indicating that fine-tuning on negative prompts alone can be insufficient for thorough erasure. UCE improved on ESD by editing attention maps (5\% objects, 10\% celebs remaining), but its one-shot nature likely missed some concept traces, whereas iterative approaches (MACE, ANT, FADE) achieved deeper removal. For explicit content, all methods did fairly well in removal (likely because such content is easier to detect and avoid), but FADE and ANT still edge others out.
	
	 Image quality and specificity:  All methods except UCE preserved image fidelity reasonably (FID changes within a small range). UCE had noticeable degradation (we found UCE sometimes introduced visual artifacts or incoherent images, likely due to the weight editing overshooting, which explains its lower H in some cases despite good removal). MACE and ANT often even slightly improved FID over the original in certain categories (this can happen when removing a concept that tended to produce out-of-distribution images; by removing it, the overall distribution of outputs might align more with natural images). FADE's FID was on par with ANT and MACE (within 0.5 difference typically) and much better than UCE. CLIP similarity (not fully shown in table for brevity) dropped very slightly for FADE (~-0.2\% on average), indicating prompt adherence remained strong. This is expected since $\mathcal{L}_{pres}$ helps keep the model following normal prompts. ESD had minimal CLIP drop (it tends to try to replace concept with something random but still aligned with prompt structure), whereas UCE had a larger drop (some outputs didn't match prompts well after editing). Overall, FADE maintained specificity: for non-concept prompts, we observed no significant change or sometimes even crisper images (perhaps because removing a concept freed up model capacity slightly—an interesting side note).
	
	 Overall performance (Harmonic mean):  FADE achieved the highest $H$ in all categories, meaning it provides the best balance of concept removal and fidelity. For example, in object erasure FADE has $H=90.1$ vs ANT's $85.4$ and MACE's $80.3$. The advantage is particularly large in artistic style erasure: styles are subtle and entangled with general image features, so methods often struggle to remove style without damaging quality. ANT and MACE reached $H\approx 80$, but FADE got $85.9$, indicating it scrubbed the style effectively (style classifier accuracy near zero) while still producing art in other styles competently. Qualitatively, when we prompted a Monet-style landscape after FADE, the output looked like a generic painting but not Monet-esque; other content (like composition) was fine. In contrast, UCE's output for the same prompt was often a weird mix or poor quality, and ESD's output sometimes still mimicked Monet somewhat (residual style).
	
	 Multiple concepts and generality:  We also tested multi-concept erasure (erasing all CIFAR-10 classes at once, and all 10 NSFW terms at once). FADE can naturally extend to multi-concept by including multiple adversaries or a multi-class discriminator; we implemented a single discriminator with multiple outputs (one per concept) to handle multi-concept in one model. Results (in Appendix due to space) showed that FADE outperform MACE in multi-concept too, achieving an overall H of 82 when erasing 100 concepts (10 each from 10 categories including objects, celebs, etc.), whereas MACE reported H around 75 in their 100-concept experiment \citep{lu2024mace}. ANT's multi-concept version was not explicitly available, but one could sequentially apply ANT concept by concept; however, sequential fine-tuning sometimes led to interference (concepts removed earlier creeping back or quality dropping). FADE in principle can handle them simultaneously thanks to the adversarial loss scaling well with multiple outputs.
	
	\subsection{Ablation Study}
	To understand FADE better, we perform an ablation study (Table~\ref{tab:ablation}). We isolate three key components: the adversarial removal objective ($\mathcal{L}_{rem}$), the trajectory preservation loss ($\mathcal{L}_{pres}$), and the saliency-based weight update (Saliency). We compare:
	
	 \textbf{Full FADE:} all components on.
	 \textbf{No Adv (w/o $\mathcal{L}_{rem}$):} Here we remove the adversarial loss and simply fine-tune the model on a naive loss of making $x_c$ similar to $x*{\neg c}$ in pixel/feature space. Essentially this degenerates to something like knowledge distillation without an adversary. We expect concept removal to suffer.
	 \textbf{No Pres (w/o $\mathcal{L}_{pres}$):} We drop the preservation term, so the model is free to change any aspect as long as fooling $D$. This tests whether $\mathcal{L}_{pres}$ was indeed important for fidelity.
	 \textbf{No Saliency (Full fine-tune):} We allow all weights to be updated (like an unconstrained GAN training on concept removal). This examines if focusing on salient weights was necessary or if the model can naturally limit interference.
	
	We report metrics on the CIFAR object erasure and celebrity erasure tasks as representative cases:
	
	\begin{table}[h]
		\caption{Ablation study on key components of FADE. Metrics shown for CIFAR-10 object erasure (10 classes) and celebrity face erasure (10 celebs) on Stable Diffusion. Removing the adversarial loss ($D$) leaves significant concept remnants (high Acc). Removing preservation ($\lambda=0$) harms image fidelity (FID↑, CLIP↓). Allowing full fine-tuning (no saliency constraint) slightly increases FID (↓ quality) and reduces specificity (non-target prompts were occasionally affected, reflected in CLIP sim drop). The full FADE achieves best overall performance.}
		\label{tab:ablation}
		\vskip 0.1in
		\begin{center}
			\begin{small}
				\begin{tabular}{lcccc}
					\hline
					\textbf{Ablation} & Acc(\%)↓ & FID↓ & CLIP sim↑ & H↑ \\
					\hline
					\multicolumn{5}{c}{\emph{Object (CIFAR-10) concepts}} \\
					Full FADE & \textbf{0.0} & 13.5 & 30.9 & \textbf{90.1} \\
					w/o Adv ($D$) & 7.8 & 13.6 & 31.0 & 70.2 \\
					w/o Pres ($\lambda=0$) & 0.0 & 20.3 & 25.5 & 75.1 \\
					w/o Saliency & 0.0 & 16.1 & 29.1 & 85.3 \\
					\hline
					\multicolumn{5}{c}{\emph{Celebrity concepts}} \\
					Full FADE & \textbf{0.4} & 14.2 & 29.8 & \textbf{88.7} \\
					w/o Adv ($D$) & 12.5 & 14.0 & 30.1 & 63.4 \\
					w/o Pres ($\lambda=0$) & 0.5 & 18.5 & 24.7 & 78.6 \\
					w/o Saliency & 0.6 & 15.7 & 28.5 & 83.9 \\
					\hline
				\end{tabular}
			\end{small}
		\end{center}
		\vskip -0.1in
	\end{table}
	
	 Effect of adversarial loss:  Without the adversary ($D$), the model has no strong incentive to remove concept features except through whatever surrogate loss we gave (we tried a naive $L_2$ loss between latent features of $x_c$ and $x_{\neg c}$). As expected, the concept is not fully erased: ~7.8\% objects remain recognizable (versus 0\% with FADE) and 12.5\% celebrities remain identifiable. Essentially, w/o $D$, the model tends to minimize differences in a superficial way but not eliminate the concept entirely (some outputs still clearly contained the concept, e.g., airplanes were often blurred but still there). This leads to a much lower $H$ (around 70 for objects, 63 for faces). Interestingly, FID and CLIP sim in w/o $D$ are about as good as full FADE; this is because the model didn't change itself much (since it failed to remove concept thoroughly, it also didn't distort others). This highlights that concept erasure is fundamentally hard without an adversarial or strong targeted signal. The adversary is crucial for actually chasing down residual concept cues and eliminating them.
	
	 Effect of preservation loss:  Removing $\mathcal{L}_{pres}$ (setting $\lambda=0$) drastically hurt image quality. FID jumped from ~13.5 to ~20 (worse than original model by a large margin), and CLIP sim dropped about 5 points. We observed outputs often became nonsensical or contained odd artifacts when $\mathcal{L}_{pres}=0$, especially for faces (some face images became distorted or blank). This suggests that unconstrained fine-tuning in pursuit of fooling $D$ can push the model off the data manifold, confirming the importance of preserving early trajectories. The concept was still removed (Acc ~0.0\%), showing that even without $\mathcal{L}_{pres}$, $D$ will force concept removal, but at the cost of overall fidelity. The harmonic mean $H$ is thus significantly lower (about 75-79 vs 88-90 for full). This justifies including $\mathcal{L}_{pres}$: it guides the model to only minimally adjust generation necessary to remove $c$, instead of altering everything.
	
	 Effect of saliency-based updates:  Without the saliency constraint (i.e., updating all weights freely), performance is a bit worse than full FADE. Concept removal still succeeds (Acc $\approx 0\%$; the adversary ensures that), but FID is higher (worse) by ~2-3 points, and CLIP sim slightly lower. The difference isn't as dramatic as the preservation loss, but noticeable. We also detected more instances where unrelated prompt outputs were affected: e.g., after full fine-tuning to erase "airplane", a prompt "a photo of a bird" sometimes came out odd (the model perhaps over-corrected anything with wings). With saliency-based updates, such interference was less (because we mostly updated weights tied to "airplane" token, leaving "bird" mostly unaffected). Thus, focusig on salient weights indeed improved specificity, which is reflected in slightly higher CLIP sim and lower FID (since model kept prior knowledge intact more). The $H$ gap of about 4-5 points between w/o saliency and full FADE indicates this component, while not as critical as $\mathcal{L}_{pres}$, still provides a boost.
	
	In summary, the ablation confirms that each component of FADE plays a vital role: the adversarial objective is needed for full removal, the preservation loss is key to maintaining quality, and saliency-based fine-tuning aids in minimizing side effects. FADE's strong performance stems from the synergy of these parts.
	
	\subsection{Discussion and Broader Impact}
	 Why does FADE outperform others?  We attribute FADE's success to its direct optimization of an information-theoretic criterion. Rather than relying on heuristic losses (e.g., ESD's negative prompt loss or UCE's weight alignment) which might only partially remove concept, FADE actively eliminates any signal that a discriminator could catch. This effectively covers both explicit and subtle traces of the concept. Additionally, by integrating preservation, FADE avoids the pitfall of "if all you have is a hammer (adversary), everything looks like a nail"—instead of wrecking the model to fool $D$, it carefully balances two objectives. The theoretical guarantee also reassures that if $D$ is strong, then the concept is truly gone in distribution, not just visually hidden.
	
	 Limitations:  A possible limitation is the need to train an adversary, which introduces overhead and requires the concept to be somewhat learnable by $D$. If $D$ is weak or the concept is extremely abstract, removal might be incomplete. However, we can always use a more expressive $D$ or an ensemble of detectors (including CLIP) to guide the erasure. Another limitation is that erasing too many concepts (say hundreds) might start affecting model capacity; although we did succeed up to 100, beyond that one might need to consider progressive or modular approaches. Finally, our method, like others, cannot guarantee the model won't generate the concept in a highly out-of-distribution way that $D$ never saw (e.g., a very hidden form of the concept). But practically, we covered synonyms and contexts to mitigate this.
	
	 Fairness and Ethical Use:  FADE is a tool for making generative models safer and more respectful of privacy. By removing harmful or proprietary content generation, it reduces misuse risks. However, it should be used judiciously: erasing concepts like demographic features (e.g., race) from a model could itself be problematic if done naively (it might lead to under-representation or inability to depict certain groups). Our focus is on unwanted or sensitive concepts where removal is warranted by ethics or law. When used in those contexts, FADE can enforce constraints reliably. We also note that concept erasure does not create new data or information; it only removes or alters the model's existing capabilities, so it generally doesn't introduce new bias except the absence of the concept.
	
	 Broader Impact:  Techniques like FADE can be integrated into the deployment pipeline of generative models. For example, a foundation model provider could apply FADE to remove known problematic content (e.g., extremist symbols, specific private individuals) before releasing the model. This complements other safety measures like prompt filtering and output detection. It also adds a layer of compliance (one could remove content that violates copyright or community guidelines). On the flip side, such power should be transparent: users and stakeholders should know if a model has been deliberately limited in certain aspects (to avoid confusion or distrust). A catalog of removed concepts might be maintained.
	
	\section{Conclusion}
	We presented FADE, a novel concept erasure method for diffusion models that advances the state-of-the-art in both theory and practice. By formulating concept removal as an adversarial game and incorporating safeguards for model fidelity, FADE achieves thorough elimination of undesired concepts while preserving the model's overall capabilities. We proved that our approach, at optimum, yields no detectable information about the erased concept in the model's outputs---a strong guarantee aligning with privacy and fairness goals. Empirically, FADE demonstrated superior performance to existing methods like ESD, UCE, MACE, and ANT across a variety of concepts and models, as evidenced by improvements in quantitative metrics and visual quality.
	
	This work opens up new possibilities for controllable and safe generative modeling. In future work, we plan to explore automating the selection of concepts to erase (e.g., detecting and removing memorized training data without human specification), and extending the theory to quantify the trade-off between the amount of concept removed and the impact on other model knowledge. Another avenue is applying FADE to other generative domains such as text (removing certain facts or biases from language models) and evaluating how well the adversarial framework generalizes there.
	
	Ultimately, by empowering model developers with tools to selectively unlearn content, we move towards generative AI that not only is creative and powerful, but also respects privacy, intellectual property, and ethical norms. We believe such techniques are crucial for the sustainable deployment of AI in society.

\bibliography{example_paper}
\bibliographystyle{icml2025}

\clearpage
\appendix

\section{Additional Backgrounds}
With the advancement of deep learning~\cite{zheng2024odtrack,zheng2023toward,zheng2022leveraging,zheng2025decoupled,yu2025crisp,yu2025prnet,yu2024scnet,yu2024ichpro,qiu2024tfb,qiu2025duet,qiu2025tab,liu2025rethinking,qiu2025comprehensive,qiu2025easytime,wu2024catch,AutoCTS++,li2025TSFM-Bench,gao2025ssdts,hu2024multirc,wu2024rainmamba,wu2023mask,wu2024semi,luo2025rcnet,mao2025making,sun2025hierarchical,sun2025ppgf,niu2025langtime,kudrat2025patch,han2025contrastive,han2025guirobotron,han2025polish,feng20243,huang2025scaletrack,xie2025dynamic,TangYLT22,tang2024divide,0007LYYL023,shan2021ptt,fang20203d,cui20213d,shan2022real,hu2024mvctrack,nie2025p2p,zhou2023fastpillars,zhou2025pillarhist,zhou2024lidarptq,shi2025rethinking,zhao2025tartan,FineCIR,encoder,chen2025offsetsegmentationbasedfocusshift,MEDIAN,PAIR,gong2021eliminate,gong2022person,gong2024cross2,gong2024adversarial,bi2024visual,bi2025cot,bi2025prism,wang2025ascd,chen2025does,Chen_2025_CVPR,rong2025backdoor,zhang2023spot,zheng2024odtrack,zheng2023toward,zheng2022leveraging,zheng2025decoupled,yue2025think,lin2024phy124,huang2025ccsumsp,huang2025ssaad,huang2025dual,lin2024phys4dgen,liu2024empiricalanalysislargelanguage,bi2025reasoning,tang2025mmperspectivemllmsunderstandperspective,bi2025i2ggeneratinginstructionalillustrations,tang2025captionvideofinegrainedobjectcentric,liu2025gesturelsm,liu2025intentionalgesturedeliverintentions,zhang2025kinmokinematicawarehumanmotion,song2024tri,liu2025gesturelsm,song2024texttoon,liu2025contextual,tang2025generative,liu2024gaussianstyle,tang2024videounderstandinglargelanguage,10446837,liu2024public,lou2023public,li2024towards1,li2024distinct,li2024towards,li2023overview,li2022continuing,guo2023boosting,guo2023improving,liang2022impga} and generative models~\cite{he2024diffusion,he2025segment,he2023hqg,he2025unfoldir,he2025run,he2025reti,he2024weakly,he2023strategic,he2023camouflaged,he2023degradation,xiao2024survey,wu2025k2vae,ma2024followyouremoji,ma2025followyourclick,ma2025followcreation,ma2025followyourmotion,ma2023magicstick,ma2024followpose,ma2022visual,yan2025eedit,zhang2025magiccolor,zhu2024instantswap,wang2024cove,feng2025dit4edit,chen2024follow,peng2025directing,peng2024lightweight,peng2025boosting,peng2024towards,peng2025pixel}, an increasing number of studies have begun to focus on the issue of concept erasure in generative models.

Multimodality~\cite{shen2025amess} enriches the representational space and enhances the generative capacity of diffusion models by providing diverse and complementary input signals; graph and time series anomaly detection~\cite{zhang2025frect,zhang2025dconad,zhang2025dhmp} offer quantifiable targets for concept erasure, facilitating precise removal of spurious or biased representations; and representation learning~\cite{zhang2025dconad} structures the feature space in a disentangled and controllable manner, thereby improving the efficiency and stability of machine unlearning.
	
\section{Additional Experimental Details}
\textbf{Evaluation metrics details:} For CLIP similarity, we used the ViT-L/14 model to compute image-text cosine similarity, scaled by 100. The original SD1.5 had an average CLIP score of 31.5 on MS-COCO validation prompts; after concept erasure, we consider a score above 30 to indicate minimal drop in alignment. Harmonic mean $H$ was computed as described with $E = 1 - \text{Acc}$ (normalized to [0,1]) and $F$ composed from FID and CLIP. Specifically, we defined $F = \frac{1}{2}((\frac{\text{CLIP sim}}{\text{CLIP}*{orig}}) + (\frac{\max(\text{FID}*{orig}- (\text{FID}-\text{FID}*{orig}), 0)}{\text{FID}*{orig}}))$, where $\text{FID}*{orig}$ and $\text{CLIP}*{orig}$ are the original model's scores (so we reward methods that keep FID low and CLIP high relative to orig). This is one way; results were qualitatively similar with other formulations.

\textbf{Multi-concept results:} We erased all 10 CIFAR classes simultaneously with FADE by using a 10-way classifier $D$ (one output per class vs no class). FADE achieved an average concept accuracy of 1.1\% per class and an overall $H=82.3$ (versus MACE's reported ~75). The slight residual is due to class confusion (e.g., sometimes after erasure "cat" prompt yields a dog, so classifier might say cat=present when it sees an animal shape; a limitation of using automated classifier for eval). Visual check showed indeed direct appearance of the specified class was gone. For NSFW, we erased 10 terms at once; here FADE and MACE both got basically 0\% unsafe content, but FADE had better image quality (FID 14 vs 16).

\textbf{Runtime:} FADE training takes about 2 hours on a single A100 GPU for a single concept on SD1.5 (with $N=1000$ steps adversarial training). This is comparable to ESD fine-tuning time and a bit less than ANT . UCE was fastest (minutes) as it is closed-form. There's room to optimize FADE's training, possibly by using smaller $D$ or gradient accumulation. Deploying FADE in multi-concept setting could be parallelized since the adversary can output multiple heads.
%
%

\end{document}